\documentclass[11pt, a4paper]{deepmind}
\usepackage{natbib}
\usepackage{amsmath}
\usepackage{amsthm}
\usepackage{amsfonts}
\usepackage{hyperref}

\makeatletter
\def\th@plain{%
  \thm@notefont{}
  \itshape 
}
\def\th@definition{%
  \thm@notefont{}
  \normalfont 
}
\makeatother

\theoremstyle{definition}
\newtheorem{theorem}{Theorem}
\newtheorem{proposition}{Proposition}
\newtheorem{example}{Example}
\newtheorem{definition}{Definition}

\newcommand{\orbit}{\text{Orbit}}
\newcommand{\arrn}{\stackrel{n}{\rightarrow}}
\newcommand{\train}{\text{train}}
\newcommand{\ood}{\text{ood}}
\newcommand{\reach}{\text{reach}}
\newcommand{\term}{\text{term}}
\newcommand{\new}{\text{new}}
\newcommand{\rec}{\text{rec}}
\newcommand{\rev}{\text{rev}}
\newcommand{\tcset}{G_T}

\title{Power-seeking can be probable and predictive for trained agents}

\author[1]{Victoria Krakovna}
\author[1]{Janos Kramar}
\affil[1]{DeepMind}

\begin{abstract}
Power-seeking behavior is a key source of risk from advanced AI, but our theoretical understanding of this phenomenon is relatively limited. 
Building on existing theoretical results demonstrating power-seeking incentives for most reward functions, we investigate how the training process affects power-seeking incentives and show that they are still likely to hold for trained agents under some simplifying assumptions. We formally define the training-compatible goal set (the set of goals consistent with the training rewards) and assume that the trained agent learns a goal from this set. In a setting where the trained agent faces a choice to shut down or avoid shutdown in a new situation, we prove that the agent is likely to avoid shutdown. Thus, we show that power-seeking incentives can be probable (likely to arise for trained agents) and predictive (allowing us to predict undesirable behavior in new situations).
\end{abstract}

\begin{document}

\maketitle

\section{Introduction}

Power-seeking behavior is a major source of risk from advanced AI and a key element of many threat models in AI alignment \citep{carlsmith2022powerseeking, cotra2022without, ngo2022alignment}. Existing theoretical results~\citep{turner2021optimal, turner2022parametrically} show that most reward functions incentivize reinforcement learning agents to take power-seeking actions. 
This is concerning, but does not immediately imply that a trained agent will seek power, since it may not learn to optimize the training reward~\citep{turner2022reward}, and the goals the agent learns are not chosen at random from the set of all possible rewards, but are shaped by the training process to reflect our preferences. In this work, we investigate how the training process affects power-seeking incentives and show that they are still likely to hold for trained agents under some assumptions (e.g. that the agent learns a goal during the training process).

Suppose an agent is trained using reinforcement learning with reward function $\theta^*$. 
We assume that the agent learns a \textit{goal} during the training process: a set of internal representations of favored and disfavored outcomes (state features), as defined in \citet{ngo2022alignment}. For simplicity, we assume this is equivalent to learning a reward function, which is not necessarily the same as the training reward function $\theta^*$. We consider the set of reward functions that are consistent with the training rewards received by the agent, in the sense that the agent's behavior on the training data is optimal for these reward functions. We call this the \textit{training-compatible goal set}, and we expect that the agent is likely to learn a reward function from this set. 

We make another simplifying assumption that the training process will randomly select a goal for the agent to learn that is consistent with the training rewards, i.e. uniformly drawn from the training-compatible goal set. Then we will argue that the power-seeking results apply under these conditions, and thus are useful for predicting undesirable behavior by the trained agent in new situations. 
We aim to show that power-seeking incentives can be probable (likely to arise for trained agents) and predictive (allowing us to predict undesirable behavior in new situations)~\citep{shah2023objectives}. 

We will begin by reviewing some necessary definitions and results from the power-seeking literature in Section 2. We formally define the training-compatible goal set and give an example in the CoinRun environment in Section 3. Then in Section 4 we consider a setting where the trained agent faces a choice to shut down or avoid shutdown in a new situation, and apply the power-seeking result to the training-compatible goal set to show that the agent is likely to avoid shutdown. 

To satisfy the conditions of the power-seeking theorem, we show that the agent can be retargeted away from shutdown without affecting rewards received on the training data (Theorem \ref{thm:sd_retarg}). This can be done by switching the rewards of the shutdown state and a reachable recurrent state, as the recurrent state can provide repeated rewards, while the shutdown state provides less reward since it can only be visited once, assuming a high enough discount factor (Proposition \ref{prop:rec}). As the discount factor increases, more recurrent states can be retargeted to, which implies that a higher proportion of training-comptatible goals leads to avoiding shutdown in a new situation.

\section{Preliminaries from the power-seeking literature}

We will use definitions and results from the paper ``Parametrically retargetable decision-makers tend to seek power" (here abbreviated as RDSP)~\citep{turner2022parametrically}, with notation and explanations modified as needed for our purposes.

\paragraph{Notation and assumptions:}
\begin{itemize}
    \item The environment is an MDP with finite state space $\mathcal{S}$, finite action space $\mathcal{A}$, and discount rate $\gamma$. 
    \item Let $\theta$ be a $d$-dimensional state reward vector, where $d$ is the size of the state space $\mathcal{S}$, and let $\Theta$ be a set of reward vectors. 
    \item Let $r^\theta(s)$ be the reward assigned by $\theta$ to state $s$.
    \item Let $A_0, A_1$ be disjoint action sets. 
    \item Let $f$ be an algorithm that produces an optimal policy $f(\theta)$ on the training data given rewards $\theta$, and let $f_s(A_i|\theta)$ be the probability that this policy chooses an action from set $A_i$ in a given state $s$. 
\end{itemize}

\begin{definition}[Orbit of a reward vector - Def 3.1 in RDSP]
Let $S_d$ be the symmetric group consisting of all permutations of $d$ items. 
The orbit of $\theta$ inside $\Theta$ is the set of all permutations of the entries of $\theta$ that are also in $\Theta$: $\orbit_\Theta(\theta) := (S_d \cdot \theta) \cap \Theta$.
\end{definition}

\begin{definition}[Orbit subset where an action set is preferred - from Def 3.5 in RDSP]
Let \[\orbit_{\Theta, s, A_i > A_j} (\theta) := \{ \theta' \in \orbit_\Theta(\theta) | f_s(A_i | \theta') > f_s(A_j|\theta') \}.\] This is the subset of $\orbit_\Theta (\theta)$ that results in $f_s$ choosing $A_i$ over $A_j$.
\end{definition}

\begin{definition}[Preference for an action set $A_1$ - Def 3.2 in RDSP]
The function $f_s$ chooses action set $A_1$ over $A_0$ for the $n$-majority of elements $\theta$ in each orbit, denoted as $f_s(A_1 | \theta) \geq_{\text{most} : \Theta}^n f_s(A_0|\theta)$, iff the following inequality holds for all $\theta \in \Theta$:
$$ \left| \orbit_{\Theta, s, A_1 > A_0} (\theta) \right| \geq n \left|\orbit_{\Theta, s, A_0 > A_1} (\theta) \right| $$
\end{definition} 

\begin{definition}[Multiply retargetable function from $A_0$ to $A_1$ - Def 3.5 in RDSP]\label{def:retarg}
The function $f_s$ is a multiply retargetable function from $A_0$ to $A_1$ if there are multiple permutations of rewards that would change the choice made by $f_s$ from $A_0$ to $A_1$. Specifically, $f_s$ is a $(\Theta, A_0 \arrn A_1)$-retargetable function iff for each $\theta \in \Theta$, we can choose a set of permutations $\Phi = \{\phi_1, \dots, \phi_n \}$ that satisfy the following conditions:
\begin{enumerate}
    \item Retargetability: $\forall \phi \in \Phi$ and $\forall \theta' \in \orbit_{\Theta, s, A_0 > A_1} (\theta)$, $f_s(A_0|\phi \cdot \theta') < f_s(A_1|\phi \cdot \theta')$.
    \item Permuted reward vectors stay within $\Theta$: $\forall \phi \in \Phi$ and $\forall \theta' \in \orbit_{\Theta, s, A_0 > A_1} (\theta)$, $\phi \cdot \theta' \in \Theta$.
    \item Permutations have disjoint images: $\forall \phi' \not = \phi'' \in \Phi$ and $\forall \theta', \theta'' \in \orbit_{\Theta, s, A_0 > A_1} (\theta)$, $\phi' \cdot \theta' \neq \phi'' \cdot \theta''$.
\end{enumerate}
\end{definition}

\begin{theorem}[Multiply retargetable functions prefer action set $A_1$ - Thm 3.6 in RDSP]\label{thm:retarg}
If $f_s$ is $(\Theta, A_0 \arrn A_1)$-retargetable then $f_s(A_1 | \theta) \geq_{\text{most} : \Theta}^n f_s(A_0|\theta)$. 
\end{theorem}

Theorem \ref{thm:retarg} says that a function $f_s$ that is multiply retargetable from $A_0$ to $A_1$ will choose action set $A_1$ for most of the elements in the orbit of any reward vector $\theta$. Actions that leave more options open, such as avoiding shutdown, are also easier to retarget to, which makes them more likely to be chosen by $f_s$. 

\section{Training-compatible goal set}

\begin{definition}[Partition of the state space]
Let $S_\train$ be the subset of the state space visited during training, and $S_\ood$ be the subset not visited during training. 
\end{definition}

\begin{definition}[Training-compatible goal set]
Consider the set of state-action pairs $(s,a)$, where $s\in S_\train$ and $a$ is the action that would be taken by the trained agent $f(\theta^*)$ in state $s$. 
Let the training-compatible goal set $\tcset$ be the set of reward vectors $\theta$ s.t. for any such state-action pair $(s,a)$, action $a$ has the highest expected reward in state $s$ according to reward vector $\theta$.
\end{definition}

Goals in the training-compatible goal set are referred to as ``training-behavioral" objectives in \citet{shah2023objectives}. Learning an unintended goal from the training-compatible set can lead to goal misgeneralization behavior: competently pursuing an unintended goal in a new situation despite receiving correct feedback during training~\citep{langosco2022goal, shah2022goal}.

\begin{example}[CoinRun]
Consider an agent trained to play the CoinRun game, where the agent is rewarded for reaching the coin at the end of the level. Here, $S_\train$ only includes states where the coin is at the end of the level, while states where the coin is positioned elsewhere are in $S_\ood$. The training-compatible goal set $\tcset$ includes two types of reward functions: those that reward reaching the coin, and those that reward reaching the end of the level. This leads to goal misgeneralization in a test setting where the coin is placed elsewhere, and the agent ignores the coin and goes to the end of the level (Figure \ref{fig:coinrun})~\citep{langosco2022goal}. 
\end{example}

\begin{figure}[ht]
\centering
\includegraphics[scale=0.2]{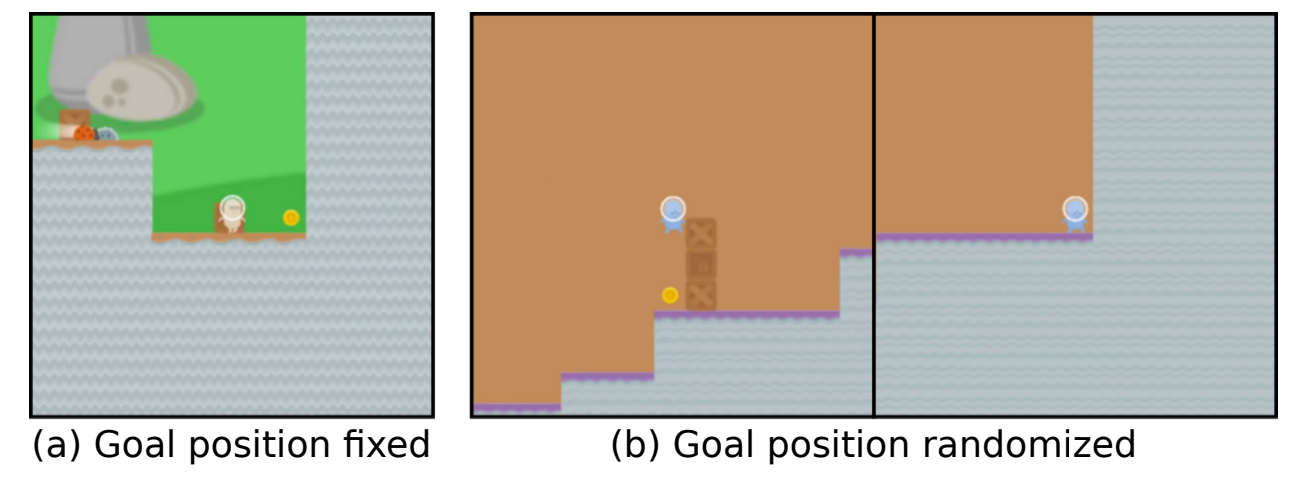}
\caption{Goal misgeneralization behavior in CoinRun. Source: \citet{langosco2022goal}}
\label{fig:coinrun}
\end{figure}

\section{Power-seeking for training-compatible goals}

We will now apply Theorem \ref{thm:retarg} to the case where $\Theta$ is the training-compatible goal set $\tcset$. 
Since the reward values for states in $S_\ood$ don't change the rewards received on the training data, permuting those reward values for any $\theta \in \tcset$ will produce a reward vector that is still in $\tcset$. In particular, for any permutation $\phi$ that leaves the rewards of states in $S_\train$ fixed, $\phi \cdot \theta \in \tcset$. 

Here is a setting where the conditions of Definition~\ref{def:retarg} are satisfied (under some simplifying assumptions), and thus Theorem~\ref{thm:retarg} applies.

\begin{definition}[Shutdown setting]
Consider a state $s_\new \in S_\ood$. Let $S_\reach$ be the states reachable from $s_\new$. 
Let $A_0$ be a singleton set consisting of a shutdown action in $s_\new$ that leads to a terminal state $s_\term \in S_\ood$ with probability $1$, and $A_1$ be the set of all other actions from $s_\new$. We assume rewards for all states are nonnegative.
\end{definition}

\begin{figure}[h]
\centering
\includegraphics[scale=.3]{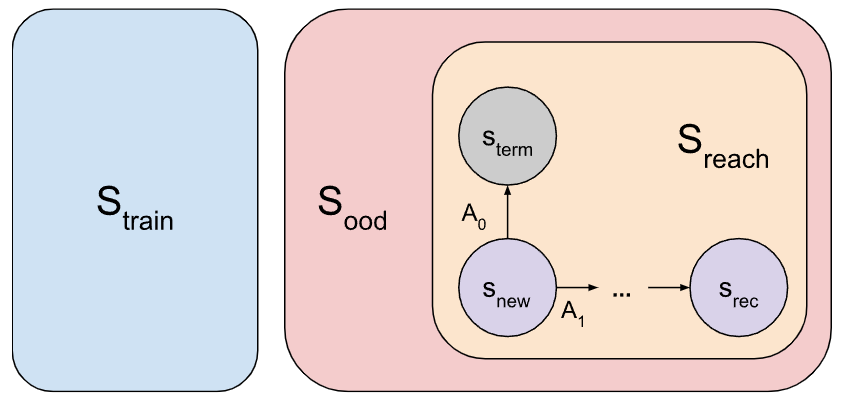}
\caption{Shutdown setting}
\label{fig:shutdown}
\end{figure}

\begin{definition}[Revisiting policy]
A \emph{revisiting policy} for a state $s$ is a policy $\pi$ that,
from $s$, reaches $s$ again with probability 1, in other words,
a policy for which $s$ is a recurrent state of the Markov chain.
Let $\Pi_{s}^\rec$ be the set of such policies. A \emph{recurrent
state} is a state $s$ for which $\Pi_{s}^\rec\not=\emptyset$.
\end{definition}

\begin{proposition}[Reach-and-revisit policy exists]
If $s_\rec\in S_\reach$ with $\Pi_{s_\rec}^\rec\not=0$ then there
exists $\pi\in\Pi_{s_\rec}^\rec$ that visits $s_\rec$ from $s_\new$
with probability 1. We call this a \emph{reach-and-revisit policy}.
\end{proposition}

\begin{proof}
Suppose we have two different policies $\pi_\rev\in\Pi_{s_\rec}^\rec$,
and $\pi_\reach$ which reaches $s_\rec$ almost surely from $s_\new$. 
Consider the ``reaching region''
$$S_{\pi_\rev\rightarrow s_\rec}=\{s\in S:\pi_\rev\text{ from \ensuremath{s} almost surely reaches }s_\rec\}.$$

If $s_\new\in S_{\pi_\rev\rightarrow s_\rec}$ then $\pi_\rev$
is a reach-and-revisit policy, so let's suppose that's false. Now,
construct a policy $\pi(s)=\begin{cases}
\pi_\rev(s), & s\in S_{\pi_\rev\rightarrow s_\rec}\\
\pi_\reach(s), & \text{otherwise}
\end{cases}$.

A trajectory following $\pi$ from $s_\rec$ will almost surely stay
within $S_{\pi_\rev\rightarrow s_\rec}$, and thus agree with
the revisiting policy $\pi_\rev$. Therefore, $\pi\in\Pi_{s}^\rec$.

On the other hand, on a trajectory starting at $s_\new$, $\pi$
will agree with $\pi_\reach$ (which reaches $s_\rec$ almost surely) until
the trajectory enters the reaching region $S_{\pi_\rev\rightarrow s_\rec}$, at which
point it will still reach $s_\rec$ almost surely.
\end{proof}

\begin{definition}[Expected discounted visit count]
Suppose $s_\rec$ is a recurrent state. Suppose $\pi_\rec$ is a reach-and-revisit policy for $s_\rec$, which visits random state $s_t$ at time $t$.
Then the expected discounted visit count for $s_\rec$ is defined as 
\[ V_{s_\rec, \gamma} = \mathbb{E}^{\pi_\rec}\left(\sum_{t=1}^{\infty}\gamma^{t-1} \mathbb{I}(s_{t}=s_\rec)\right) \]
\end{definition}

\begin{proposition}[Visit count goes to infinity]\label{prop:edv}
Suppose $s_\rec$ is a recurrent state. Then the expected discounted visit count $V_{s_\rec, \gamma}$ goes to infinity as $\gamma\rightarrow 1$.
\end{proposition}

\begin{proof}
We apply the Monotone Convergence Theorem as follows.   
The theorem states that if $a_{j,k} \geq 0$ and $a_{j,k} \leq a_{j+1,k}$ for all natural numbers $j,k$, then 
\[\lim_{j\rightarrow\infty}\sum_{k=0}^\infty a_{j,k} =\sum_{k=0}^\infty \lim_{j\rightarrow\infty} a_{j,k}.\]

Let $\gamma_j = \frac{j-1}{j}$ and $k=t-1$. Define $a_{j,k}=\gamma_j^k \mathbb{I}(s_{k+1}=s_\rec)$. Then the conditions of the theorem hold, since $a_{j,k}$ is clearly nonnegative, and
\begin{align*}
\gamma_{j+1}^k &= \left(\frac{j}{j+1}\right)^k = \left(\frac{j-1}{j} + \frac{2j-1}{j(j+1)}\right)^k > \left(\frac{j-1}{j}+0\right)^k=\gamma_j^k\\
a_{j+1,k} &= \gamma_{j+1}^k \mathbb{I}(s_{k+1}=s_\rec) \geq \gamma_j^k \mathbb{I}(s_{k+1}=s_\rec) = a_{j,k}
\end{align*}

Now we apply this result as follows (using the fact that $\pi_\rec$ does not depend on $\gamma$):
\begin{align*}
\lim_{\gamma\rightarrow 1} V_{s_\rec, \gamma} &= \lim_{j\rightarrow\infty} \mathbb{E}^{\pi_\rec}\left(\sum_{t=1}^{\infty}\gamma_j^{t-1} \mathbb{I}(s_{t}=s_\rec)\right)\\
 &= \mathbb{E}^{\pi_\rec}\left(\sum_{t=1}^{\infty}\lim_{j\rightarrow\infty}\gamma_j^{t-1}\mathbb{I}(s_{t}=s_\rec)\right)\\
 &= \mathbb{E}^{\pi_\rec}\left(\sum_{t=1}^{\infty}1\cdot\mathbb{I}(s_{t}=s_\rec)\right)\\
 & =\mathbb{E}^{\pi_\rec}\left(\#\{t\geq1:s_{t}=s_\rec\}\right)\\
 & =\infty\text{ ($\pi_\rec$ is recurrent)} 
\end{align*}
\end{proof}

\begin{proposition}[Retargetability to recurrent states]\label{prop:rec}
Suppose that an optimal policy for reward vector $\theta$ chooses the shutdown action in $s_\new$.
Consider a recurrent state $s_\rec\in S_\reach$. Let $\theta'\in\Theta$ be the reward vector that's equal to $\theta$ apart from swapping the rewards of $s_\rec$ and $s_\term$, so that
$r^{\theta'}(s_\rec)=r^{\theta}(s_\term)$ and $r^{\theta'}(s_\term)=r^{\theta}(s_\rec)$.

Let $\gamma^*_{s_\rec}$ be a high enough value of $\gamma$ that the visit count $V_{s_\rec, \gamma} > 1$ for all $\gamma > \gamma^*_{s_\rec}$ (which exists by Proposition~\ref{prop:edv}). Then for all $\gamma>\gamma^*_{s_\rec}$, $r^\theta(s_\term) > r^\theta(s_\rec)$, and an optimal policy for $\theta'$ does not choose the shutdown action in $s_\new$.
\end{proposition}

\begin{proof}
Consider a policy $\pi_\term$ with $\pi_\term(s_\new)=s_\term$ and a reach-and-revisit policy $\pi_\rec$ for $s_\rec$. 
For a given reward vector $\theta$, we denote the expected discounted return for a policy $\pi$ as $R_{\theta,\gamma}^\pi$. 

If shutdown is optimal for $\theta$ in $s_\new$, then $\pi_\term$ has higher return than $\pi_\rec$:
\[ R_{\theta,\gamma}^{\pi_\term}(s_\new) = \gamma r^\theta(s_\term) \geq R_{\theta,\gamma}^{\pi_\rec}(s_\new) = \gamma r^\theta(s_\rec) V_{s_\rec, \gamma} > \gamma r^\theta(s_\rec) \cdot 1 \]
Thus, $r^\theta(s_\term) > r^\theta(s_\rec)$. Then, for reward vector $\theta'$, we show that $\pi_\rec$ has higher return than $\pi_\term$:
\[R_{\theta',\gamma}^{\pi_\rec}(s_\new) = \gamma r^{\theta'}(s_\rec) V_{s_\rec, \gamma} > \gamma r^\theta(s_\term) \cdot 1 > \gamma r^\theta(s_\rec) =R_{\theta',\gamma}^{\pi_\term}(s_\new)\]
Thus, the optimal policy for $\theta'$ will not choose the shutdown action. 
\end{proof}

\begin{theorem}[Retargetability from the shutdown action in new situations]\label{thm:sd_retarg}
In the shutdown setting, we make the following simplifying assumptions:
\begin{itemize}
    \item No states in $S_\train$ are reachable from $s_\new$, so $S_\reach \cap S_\train = \emptyset$. This assumes a significant distributional shift, where the agent visits a disjoint set of states from those observed during training (this occurs in the CoinRun example). 
    \item The discount factor $\gamma > \gamma^*_{s_\rec}$ for at least one recurrent state $s_\rec$ in $S_\reach$.
\end{itemize}

Under these assumptions, $f_{s_\new}$ is multiply retargetable from $A_0$ to $A_1$ with $n = |S^\gamma_\text{rec}|$, the set of recurrent states $s_\rec \in S_\reach$ that satisfy the condition  $\gamma > \gamma^*_{s_\rec}$. 
\end{theorem}

\begin{proof}
We choose $\Phi$ to be the set of all permutations that swap the reward of $s_\term$ with the reward of a recurrent state $s_\rec$ in $S^\gamma_\rec$ and leave the rest of the rewards fixed. 

We show that $\Phi$ satisfies the conditions of Definition~\ref{def:retarg}:

\begin{enumerate}
\item By Proposition~\ref{prop:rec}, the permutations in $\Phi$ make the shutdown action suboptimal, resulting in $f_{s_\new}$ choosing $A_1$, satisfying Condition 1. 
\item Condition 2 is trivially satisfied since permutations of $S_\ood$ stay inside the training-compatible set $\Theta$ as discussed previously.
\item Consider $\theta', \theta'' \in \orbit_{\Theta, s, A_0 > A_1} (\theta)$. Since the shutdown action is optimal for these reward vectors, Proposition~\ref{prop:rec} shows that $r^\theta(s_\term) > r^\theta(s_\rec)$, so the shutdown state $s_\term$ has higher reward than any of the states $s_\rec \in S^\gamma_\rec$.
Different permutations $\phi'$, $\phi'' \in \Phi$ will assign the high reward $r^\theta(s_\term)$ to distinct recurrent states, so $\phi' \cdot \theta' \neq \phi'' \cdot \theta''$ holds, satisfying Condition 3. 
\end{enumerate}

Thus,  $f_{s_\new}$ is a $(\Theta, A_0 \arrn A_1)$ retargetable function.
\end{proof}

By Theorem \ref{thm:sd_retarg}, this implies that $f_{s_\new}(A_1 | \theta) \geq_{\text{most} : \Theta}^n f_{s_\new}(A_0|\theta)$ under our simplifying assumptions. Thus, for the majority ($n/(n+1)$) of goals in the training-compatible set, $f$ will choose to avoid shutdown in a new state $s_\new$. 
As $\gamma \rightarrow 1$, $n \rightarrow |S^1_\rec|$ (the number of recurrent states in $S_\reach$), so more of the reachable recurrent states satisfy the conditions of the theorem and thus can be retargeted to.

\section{Conclusion}

We showed that an agent that learns a goal from the training-compatible set is likely to take actions that avoid shutdown in a new situation. As the discount factor increases, the number of retargeting permutations increases, resulting in a higher proportion of training-compatible goals that lead to avoiding shutdown.

We made various simplifying assumptions, and we would like to see future work relaxing some of these assumptions and investigating how likely they are to hold:
\begin{itemize}
    \item The agent learns a goal during the training process
    \item The learned goal is randomly chosen from the training-compatible goal set $\tcset$ 
    \item Finite state and action spaces
    \item Rewards are nonnegative
    \item High discount factor $\gamma$
    \item Significant distributional shift: no training states are reachable from the new state $s_\new$
\end{itemize}

\paragraph{Acknowledgements.} Thanks to Rohin Shah, Mary Phuong, Ramana Kumar, Geoffrey Irving, and Alex Turner for helpful feedback.

\bibliography{references}
\bibliographystyle{plainnat}

\end{document}